\documentclass{article}

\usepackage{microtype}
\usepackage{graphicx}
\usepackage{subfigure}
\usepackage{booktabs} 

\usepackage{amsmath}
\usepackage{amssymb}
\usepackage{amsthm}
\usepackage{multirow}
\usepackage{algorithm}
\usepackage{algorithmic}
\newtheorem{theorem}{Theorem}
\newtheorem{proposition}{Proposition}
\newtheorem{lemma}{Lemma}

\usepackage{hyperref}

\usepackage[accepted]{icml2018}

\icmltitlerunning{GAIN: Missing Data Imputation using Generative Adversarial Nets}

\begin{document}

\twocolumn[
\icmltitle{GAIN: Missing Data Imputation using Generative Adversarial Nets}



\icmlsetsymbol{equal}{*}

\begin{icmlauthorlist}
\icmlauthor{Jinsung Yoon}{to,equal}
\icmlauthor{James Jordon}{goo,equal}
\icmlauthor{Mihaela van der Schaar}{to,goo,ati}
\end{icmlauthorlist}

\icmlaffiliation{to}{University of California, Los Angeles, CA, USA}
\icmlaffiliation{goo}{University of Oxford, UK}
\icmlaffiliation{ati}{Alan Turing Institute, UK}

\icmlcorrespondingauthor{Jinsung Yoon}{jsyoon0823@gmail.com}


\vskip 0.3in
]



\printAffiliationsAndNotice{\icmlEqualContribution} 

\begin{abstract}
We propose a novel method for imputing missing data by adapting the well-known Generative Adversarial Nets (GAN) framework. Accordingly, we call our method Generative Adversarial Imputation Nets (GAIN). The generator ($G$) observes some components of a real data vector, imputes the missing components conditioned on what is actually observed, and outputs a completed vector. The discriminator ($D$) then takes a completed vector and attempts to determine which components were actually observed and which were imputed. To ensure that $D$ forces $G$ to learn the desired distribution, we provide $D$ with some additional information in the form of a {\em hint} vector. The hint reveals to $D$ {\em partial} information about the missingness of the original sample, which is used by $D$ to focus its attention on the imputation quality of particular components. This hint ensures that $G$ does in fact learn to generate according to the true data distribution. We tested our method on various datasets and found that GAIN significantly outperforms state-of-the-art imputation methods.

\end{abstract}

\section{Introduction}\label{sec:introduction}
Missing data is a pervasive problem. Data may be missing because it was never collected, records were lost or for many other reasons. In the medical domain, the respiratory rate of a patient may not have been measured (perhaps because it was deemed unnecessary/unimportant) or accidentally not recorded \cite{yoon_jbhi,ahmed_tbme}. It may also be the case that certain pieces of information are difficult or even dangerous to acquire (such as information gathered from a biopsy), and so these were not gathered for those reasons \cite{yoon_plosone}. An imputation algorithm can be used to estimate missing values based on data that was observed/measured, such as the systolic blood pressure and heart rate of the patient \cite{yoon_deep}. A substantial amount of research has been dedicated to developing imputation algorithms for medical data \cite{Medimpute1,Medimpute2,Medimpute3,Medimpute4}. Imputation algorithms are also used in many other applications such as image concealment, data compression, and counterfactual estimation \cite{Rubin,Missing_Book,yoon_ganite}. 

Missing data can be categorized into three types: (1) the data is missing completely at random (MCAR) if the missingness occurs entirely at random (there is no dependency on any of the variables), (2) the data is missing at random (MAR) if the missingness depends only on the {\em observed} variables\footnote{A formal definition of MAR can be found in the Supplementary Materials.}, (3) the data is missing not at random (MNAR) if the missingness is neither MCAR nor MAR (more specifically, the data is MNAR if the missingness depends on {\em both observed} variables and the {\em unobserved} variables; thus, missingness cannot be fully accounted for by the observed variables). In this paper we provide theoretical results for our algorithm under the MCAR assumption, and compare to other state-of-the-art methods in this setting\footnote{Empirical results for the MAR and MNAR settings are shown in the Supplementary Materials.}.

State-of-the-art imputation methods can be categorized as either discriminative or generative. Discriminative methods include MICE \cite{MICE,MICE-R}, MissForest \cite{missforest}, and matrix completion \cite{Mat-0,Mat-1,Mat-2,Mat-3}; generative methods include algorithms based on Expectation Maximization \cite{EM} and algorithms based on deep learning (e.g. denoising autoencoders (DAE) and generative adversarial nets (GAN)) \cite{DAE,autoencoder,GAN-imagecomplete}. However, current generative methods for imputation have various drawbacks. For instance, the approach for data imputation based on \cite{EM} makes assumptions about the underlying distribution and fails to generalize well when datasets contain mixed categorical and continuous variables. In contrast, the approaches based on DAE \cite{DAE} have been shown to work well in practice but require complete data during training. In many circumstances, missing values are part of the inherent structure of the problem so obtaining a complete dataset is impossible. Another approach with DAE \cite{autoencoder} allows for an incomplete dataset; however, it only utilizes the observed components to learn the representations of the data. \cite{GAN-imagecomplete} uses Deep Convolutional GANs for image completion; however, it also requires complete data for training the discriminator.

In this paper, we propose a novel imputation method, which we call Generative Adversarial Imputation Nets (GAIN), that generalizes the well-known GAN \cite{GAN} and is able to operate successfully even when complete data is unavailable. In GAIN, the generator's goal is to accurately impute missing data, and the discriminator's goal is to distinguish between observed and imputed components. The discriminator is trained to minimize the classification loss (when classifying which components were observed and which have been imputed), and the generator is trained to maximize the discriminator's misclassification rate. Thus, these two networks are trained using an adversarial process. To achieve this goal, GAIN builds on and adapts the standard GAN architecture. To ensure that the result of this adversarial process is the desired target, the GAIN architecture provides the discriminator with additional information in the form of ``hints". This hinting ensures that the generator generates samples according to the true underlying data distribution.

    \begin{figure}[t!]
        \centering
        \includegraphics[width=0.48\textwidth]{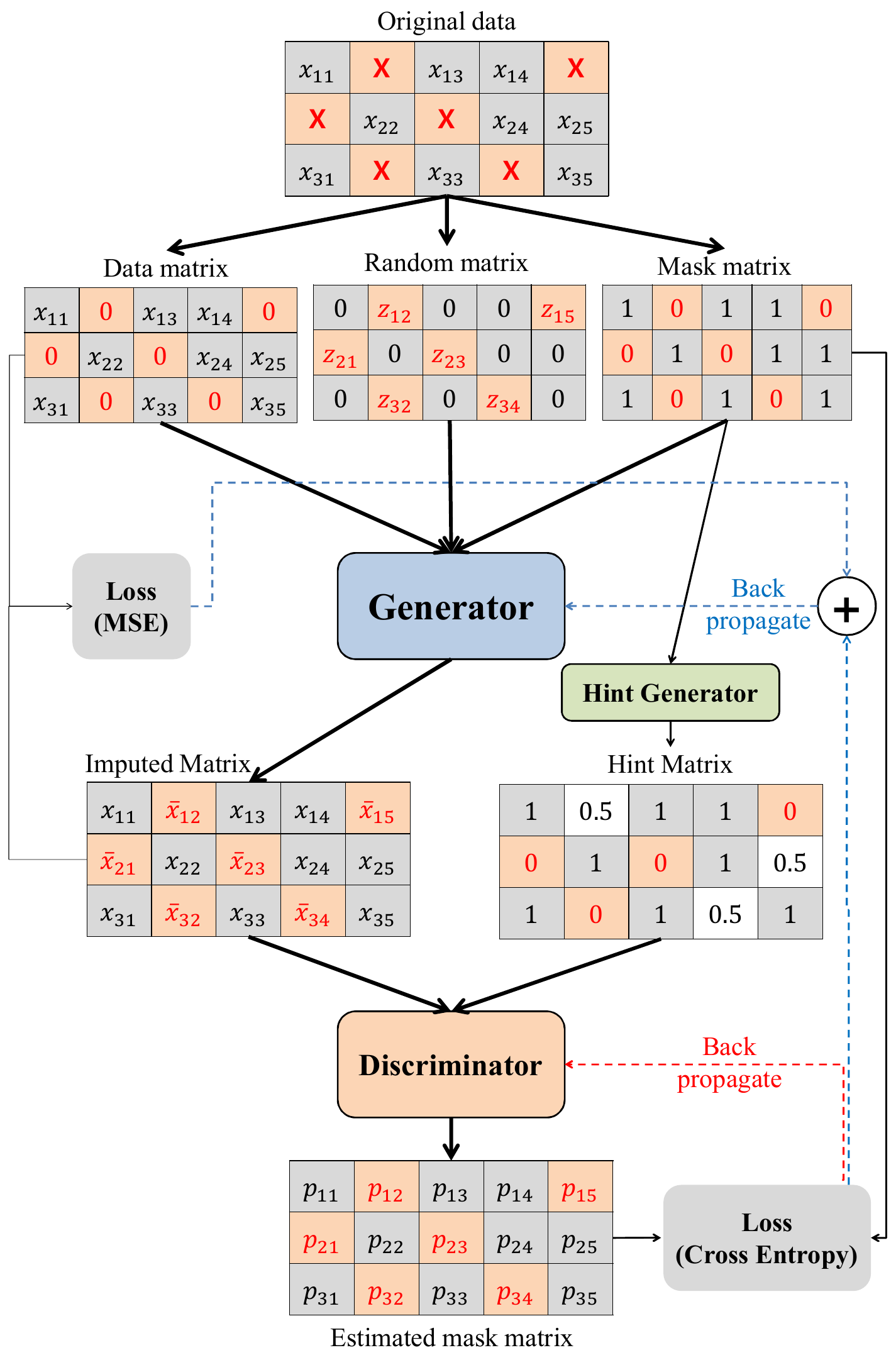}
        \caption{The architecture of GAIN}
        \label{fig:model}
    \end{figure}
    
\section{Problem Formulation} \label{sect:problem_formulate}
Consider a $d$-dimensional space $\mathcal{X} = \mathcal{X}_1 \times ... \times \mathcal{X}_d$. Suppose that $\mathbf{X} = (X_1, ..., X_d)$ is a random variable (either continuous or binary) taking values in $\mathcal{X}$, whose distribution we will denote $P(\mathbf{X})$. Suppose that $\mathbf{M} = (M_1, ..., M_d)$ is a random variable taking values in $\{0, 1\}^d$. We will call $\mathbf{X}$ the data vector, and $\mathbf{M}$ the mask vector.  

For each $i \in \{1, ..., d\}$, we define a new space $\tilde{\mathcal{X}_i} = \mathcal{X}_i \cup \{*\}$ where $*$ is simply a point not in any $\mathcal{X}_i$, representing an unobserved value. Let $\tilde{\mathcal{X}} = \tilde{\mathcal{X}}_1 \times ... \times \tilde{\mathcal{X}}_d$.  We define a new random variable $\tilde{\mathbf{X}} = (\tilde{X}_1, ..., \tilde{X}_d) \in \tilde{\mathcal{X}}$ in the following way:
\begin{equation} \label{eq:xtilde}
\tilde{X}_i = 
\begin{cases}
X_i, & \text{if } M_i = 1\\
*, & \text{otherwise}
\end{cases}
\end{equation}
so that $\mathbf{M}$ indicates which components of $\mathbf{X}$ are observed. Note that we can recover $\mathbf{M}$ from $\tilde{\mathbf{X}}$.

Throughout the remainder of the paper, we will often use lower-case letters to denote realizations of a random variable and use the notation $\mathbf{1}$ to denote a vector of $1$s, whose dimension will be clear from the context (most often, $d$).

\subsection{Imputation}
In the imputation setting, $n$ i.i.d. copies of $\tilde{\mathbf{X}}$ are realized, denoted $\tilde{\mathbf{x}}^1, ..., \tilde{\mathbf{x}}^n$ and we define the dataset $\mathcal{D} = \{(\tilde{\mathbf{x}}^i, \mathbf{m}^i)\}_{i=1}^n$, where $\mathbf{m}^i$ is simply the recovered realization of $\mathbf{M}$ corresponding to $\tilde{\mathbf{x}}^i$. 

Our goal is to {\em impute} the unobserved values in each $\tilde{\mathbf{x}}_i$. Formally, we want to generate samples according to $P(\mathbf{X} | \tilde{\mathbf{X}} = \tilde{\mathbf{x}}^i)$, the conditional distribution of $\mathbf{X}$ given $\tilde{\mathbf{X}} = \tilde{\mathbf{x}}^i$, for each $i$, to fill in the missing data points in $\mathcal{D}$. By attempting to model the {\em distribution} of the data rather than just the expectation, we are able to make multiple draws and therefore make {\em multiple imputations} allowing us to capture the uncertainty of the imputed values \cite{MICE, MICE-R, Rubin}.

\section{Generative Adversarial Imputation Nets}\label{sect:gain}
In this section we describe our approach for simulating $P(\mathbf{X} | \tilde{\mathbf{X}} = \tilde{\mathbf{x}}^i)$ which is motivated by GANs. We highlight key similarities and differences to a standard (conditional) GAN throughout. Fig. \ref{fig:model} depicts the overall architecture.
\subsection{Generator}
The generator, $G$, takes (realizations of) $\tilde{\mathbf{X}}$, $\mathbf{M}$ and a noise variable, $\mathbf{Z}$, as input and outputs $\bar{\mathbf{X}}$, a vector of imputations. Let $G: \tilde{\mathcal{X}} \times \{0, 1\}^d \times [0, 1]^d \to \mathcal{X}$ be a function, and $\mathbf{Z} = (Z_1, ..., Z_d)$ be d-dimensional noise (independent of all other variables).

Then we define the random variables $\bar{\mathbf{X}}, \hat{\mathbf{X}} \in \mathcal{X}$ by
\begin{align} \label{eq:xbar}
\bar{\mathbf{X}} &= G(\tilde{\mathbf{X}}, \mathbf{M}, (\mathbf{1 - M}) \odot \mathbf{Z}) \\ \label{eq:xhat}
\hat{\mathbf{X}} &= \mathbf{M} \odot \tilde{\mathbf{X}} + (\mathbf{1} - \mathbf{M}) \odot \bar{\mathbf{X}}
\end{align}
where $\odot$ denotes element-wise multiplication. $\bar{\mathbf{X}}$ corresponds to the vector of {\em imputed} values (note that $G$ outputs a value for every component, even if its value was observed) and $\hat{\mathbf{X}}$ corresponds to the completed data vector, that is, the vector obtained by taking the partial observation $\tilde{\mathbf{X}}$ and replacing each $*$ with the corresponding value of $\bar{\mathbf{X}}$.

This setup is very similar to a standard GAN, with $\mathbf{Z}$ being analogous to the noise variables introduced in that framework. Note, though, that in this framework, the target distribution, $P(\mathbf{X} | \tilde{\mathbf{X}})$, is essentially $||\mathbf{1 - M}||_1$-dimensional and so the noise we pass into the generator is $(\mathbf{1 - M}) \odot \mathbf{Z}$, rather than simply $\mathbf{Z}$, so that its dimension matches that of the targeted distribution.

\subsection{Discriminator}
As in the GAN framework, we introduce a discriminator, $D$, that will be used as an adversary to train $G$. However, unlike in a standard GAN where the output of the generator is either {\em completely} real or {\em completely} fake, in this setting the output is comprised of some components that are real and some that are fake. Rather than identifying that an entire vector is real or fake, the discriminator attempts to distinguish which {\em components} are real (observed) or fake (imputed) - this amounts to predicting the mask vector, $\mathbf{m}$. Note that the mask vector $\mathbf{M}$ is pre-determined by the dataset.

Formally, the discriminator is a function $D: \mathcal{X} \to [0, 1]^d$ with the $i$-th component of $D(\hat{\mathbf{x}})$ corresponding to the probability that the $i$-th component of $\hat{\mathbf{x}}$ was observed.

\subsection{Hint} \label{subsec:hint}
As will be seen in the theoretical results that follow, it is necessary to introduce what we call a hint mechanism. A hint mechanism is a random variable, $\mathbf{H}$, taking values in a space $\mathcal{H}$, both of which {\em we define}. We allow $\mathbf{H}$ to depend on $\mathbf{M}$ and for each (imputed) sample $(\hat{\mathbf{x}}, \mathbf{m})$, we draw $\mathbf{h}$ according to the distribution $\mathbf{H} | \mathbf{M} = \mathbf{m}$. We pass $\mathbf{h}$ as an additional input to the discriminator and so it becomes a function $D: \mathcal{X} \times \mathcal{H} \to [0, 1]^d$, where now the $i$-th component of $D(\hat{\mathbf{x}}, \mathbf{h})$ corresponds to the probability that the $i$-th component of $\hat{\mathbf{x}}$ was observed conditional on $\hat{\mathbf{X}} = \hat{\mathbf{x}}$ {\em and} $\mathbf{H} = \mathbf{h}$.

By defining $\mathbf{H}$ in different ways, we control the amount of information contained in $\mathbf{H}$ about $\mathbf{M}$ and in particular we show (in Proposition \ref{prop:nonunique}) that if we do not provide ``enough" information about $\mathbf{M}$ to $D$ (such as if we simply did not have a hinting mechanism), then there are several distributions that $G$ could reproduce that would all be optimal with respect to $D$.

\subsection{Objective}
We train $D$ to {\em maximize} the probability of correctly predicting $\mathbf{M}$. We train $G$ to {\em minimize} the probability of $D$ predicting $\mathbf{M}$. We define the quantity $V(D, G)$ to be

\begin{align} \label{eq:V}
V(D, G) &= \mathbb{E}_{\hat{\mathbf{X}}, \mathbf{M}, \mathbf{H}}\Big[\mathbf{M}^T \log D(\hat{\mathbf{X}}, \mathbf{H}) \\\nonumber
&\quad+ (\mathbf{1 - M})^T \log\big(\mathbf{1} - D(\hat{\mathbf{X}}, \mathbf{H})\big)\Big],
\end{align}
where $\log$ is element-wise logarithm and dependence on $G$ is through $\hat{\mathbf{X}}$.

Then, as with the standard GAN, we define the objective of GAIN to be the minimax problem given by
\begin{equation}\label{eq:obj}
\min_{G} \max_{D} V(D, G).
\end{equation}

We define the loss function $\mathcal{L} : \{0,1\}^d \times [0, 1]^d \to \mathbb{R}$ by
\begin{equation} \label{eq:loss}
\mathcal{L}(\mathbf{a},\mathbf{b})=\sum_{i=1}^{d}\Big[a_i\log(b_i)+(1-a_{i})\log(1-b_i)\Big].
\end{equation}
Writing $\hat{\mathbf{M}} = D(\hat{\mathbf{X}}, \mathbf{H})$, we can then rewrite (\ref{eq:obj}) as
\begin{equation} \label{eq:reform}
\min_G\max_D \mathbb{E}\big[\mathcal{L}(\mathbf{M}, \hat{\mathbf{M}})\big].
\end{equation}

\section{Theoretical Analysis} \label{sec:theory}
In this section we provide a theoretical analysis of (\ref{eq:obj}). Given a d-dimensional space $\mathcal{Z} = \mathcal{Z}_1 \times ... \times \mathcal{Z}_d$, a (probability) density\footnote{For ease of exposition, we use the term density even when referring to a probability mass function.} $p$ over $\mathcal{Z}$ corresponding to a random variable $Z$, and a vector $\mathbf{b} \in \{0, 1\}^d$ we define the set $A_{\mathbf{b}} = \{i : b_i = 1\}$, the projection $\phi_{\mathbf{b}} : \mathcal{Z} \to \Pi_{i \in A_{\mathbf{b}}} \mathcal{Z}_i$ by $\phi_{\mathbf{b}}(z) = (z_i)_{i \in A}$ and the density $p^{\mathbf{b}}$ to be the density of $\phi_{\mathbf{b}}(Z)$.

Throughout this section, we make the assumption that $\mathbf{M}$ is independent of $\mathbf{X}$, i.e. that the data is MCAR.

We will write $p(\mathbf{x}, \mathbf{m}, \mathbf{h})$ to denote the density of the random variable $(\hat{\mathbf{X}}, \mathbf{M}, \mathbf{H})$ and we will write $\hat{p}$, $p_m$ and $p_h$ to denote the marginal densities (of $p$) corresponding to $\hat{\mathbf{X}}$, $\mathbf{M}$ and $\mathbf{H}$, respectively. When referring to the joint density of two of the three variables (potentially conditioned on the third), we will simply use $p$, abusing notation slightly.

It is more intuitive to think of this density through its decomposition into densities corresponding to the true data generating process, and to the generator defined by (\ref{eq:xbar}),
\begin{align} \label{eq:decomp}
p(\mathbf{x}, \mathbf{m}, \mathbf{h}) =& p_m(\mathbf{m})\hat{p}^{\mathbf{m}}(\phi_{\mathbf{m}}(\mathbf{x} | \mathbf{m}))\\\nonumber
&\times \hat{p}^{\mathbf{1 - m}}(\phi_{\mathbf{1 - m}}(\mathbf{x}) | \mathbf{m}, \phi_{\mathbf{m}}(\mathbf{x}))p_h(\mathbf{h} | \mathbf{m}).
\end{align}
The first two terms in (\ref{eq:decomp}) are both defined by the data, where $\hat{p}^{\mathbf{m}}(\phi_{\mathbf{m}}(\mathbf{x}) | \mathbf{m})$ is the density of $\phi_{\mathbf{m}}(\hat{\mathbf{X}}) | \mathbf{M} = \mathbf{m}$ which corresponds to the density of $\phi_{\mathbf{m}}(\mathbf{X})$ (i.e. the true data distribution), since conditional on $\mathbf{M} = \mathbf{m}$, $\phi_{\mathbf{m}}(\hat{\mathbf{X}}) = \phi_{\mathbf{m}}(\mathbf{X})$ (see equations \ref{eq:xtilde} and \ref{eq:xhat}). The third term, $\hat{p}^{\mathbf{1 - m}}(\phi_{\mathbf{1 - m}}(\mathbf{x}) | \mathbf{m}, \phi_{\mathbf{m}}(\mathbf{x}))$, is determined by the generator, $G$, and is the density of the random variable $\phi_{\mathbf{1 - m}}(G(\tilde{\mathbf{x}}, \mathbf{m}, \mathbf{Z})) = \phi_{\mathbf{1 - m}}(\bar{\mathbf{X}}) | \tilde{\mathbf{X}} = \tilde{\mathbf{x}}, \mathbf{M} = \mathbf{m}$ where $\tilde{\mathbf{x}}$ is determined by $\mathbf{m}$ and $\phi_{\mathbf{m}}(\mathbf{x})$. The final term is the conditional density of the hint, which we are free to define (its selection will be motivated by the following analysis).

Using this decomposition, one can think of drawing a sample from $\hat{p}$ as first sampling $\mathbf{m}$ according to $p_m(\cdot)$, then sampling the ``observed" components, $\mathbf{x}_{obs}$, according to $\hat{p}^{\mathbf{m}}(\cdot)$ (we can then construct $\tilde{\mathbf{x}}$ from $\mathbf{x}_{obs}$ and $\mathbf{m}$), then {\em generating} the imputed values, $\mathbf{x}_{imp}$, from the generator according to $\hat{p}^{\mathbf{1 - m}}(\cdot | \mathbf{m}, \mathbf{x}_{obs})$ and finally sampling the hint according to $p_h(\cdot |\mathbf{m})$.

\begin{lemma} \label{lem:D}
	Let $\mathbf{x} \in \mathcal{X}$. Let $p_h$ be a fixed density over the hint space $\mathcal{H}$ and let $\mathbf{h} \in \mathcal{H}$ be such that $p(\mathbf{x}, \mathbf{h}) > 0$.	Then for a fixed generator, $G$, the $i$-th component of the optimal discriminator, $D^*(\mathbf{x}, \mathbf{h})$ is given by
	\begin{align} \label{eq:optdis}
	D^*(\mathbf{x}, \mathbf{h})_i &= \frac{p(\mathbf{x}, \mathbf{h}, m_i = 1)}{p(\mathbf{x}, \mathbf{h}, m_i = 1) + p(\mathbf{x}, \mathbf{h}, m_i = 0)} \\
	&= p_m(m_i = 1 | \mathbf{x}, \mathbf{h})
	\end{align}
	for each $i \in \{1, ..., d\}$.
\end{lemma}
\begin{proof}
	All proofs are provided in Supplementary Materials.
\end{proof}

We now rewrite (\ref{eq:V}), substituting for $D^*$, to obtain the following minimization criterion for $G$:
\begin{align}
C(G) = &\mathbb{E}_{\hat{\mathbf{X}}, \mathbf{M}, \mathbf{H}}\Big(\underset{i: M_i = 1}{\sum}\log p_m(m_i = 1 | \hat{\mathbf{X}}, \mathbf{H})\\\nonumber
&+ \underset{i: M_i = 0}{\sum} \log p_m(m_i = 0 | \hat{\mathbf{X}}, \mathbf{H})\Big),
\end{align}
where dependence on $G$ is through $p_m(\cdot | \hat{\mathbf{X}})$.

\begin{theorem} \label{thm:main}
 	A global minimum for $C(G)$ is achieved if and only if the density $\hat{p}$ satisfies
	\begin{equation} \label{eq:phat}
	\hat{p}(\mathbf{x} | \mathbf{h}, m_i = t) = \hat{p}(\mathbf{x} | \mathbf{h})
	\end{equation}
	for each $i \in \{1, ..., d\}$, $\mathbf{x} \in \mathcal{X}$ and $\mathbf{h} \in \mathcal{H}$ such that $p_h(\mathbf{h} | m_i = t) > 0$.
\end{theorem}

The following proposition asserts that if $\mathbf{H}$ does not contain ``enough" information about $\mathbf{M}$, we cannot guarantee that $G$ learns the desired distribution (the one uniquely defined by the (underlying) data).

\begin{proposition} \label{prop:nonunique}
	There exist distributions of $\mathbf{X}$, $\mathbf{M}$ and $\mathbf{H}$ for which solutions to (\ref{eq:phat}) are not unique. In fact, if $\mathbf{H}$ is independent of $\mathbf{M}$, then (\ref{eq:phat}) does not define a unique density, in general.
\end{proposition}

Let the random variable $\mathbf{B} = (B_1, ..., B_d) \in \{0, 1\}^d$ be defined by first sampling $k$ from $\{1, ..., d\}$ uniformly at random and then setting 
\begin{equation} \label{eq:bdef}
B_j = \begin{cases}
1 \text{ if } j \neq k \\
0 \text{ if } j = k.
\end{cases}
\end{equation}
Let $\mathcal{H} = \{0, 0.5, 1\}^d$ and, given $\mathbf{M}$, define
\begin{equation} \label{eq:hdef}
\mathbf{H} = \mathbf{B} \odot \mathbf{M} + 0.5(\mathbf{1 - B}).
\end{equation}
Observe first that $\mathbf{H}$ is such that $H_i = t \implies M_i = t$ for $t \in \{0, 1\}$ but that $H_i = 0.5$ implies nothing about $M_i$. In other words, $\mathbf{H}$ reveals all but one of the components of $\mathbf{M}$ to $D$. Note, however, that $\mathbf{H}$ does contain some information about $M_i$ since $M_i$ is not assumed to be independent of the other components of $\mathbf{M}$.

The following lemma confirms that the discriminator behaves as we expect with respect to this hint mechanism.

\begin{lemma} \label{lem:DH}
	Suppose $\mathbf{H}$ is defined as above. Then for $\mathbf{h}$ such that $h_i = 0$ we have $D^*(\mathbf{x}, \mathbf{h})_i = 0$ and for $\mathbf{h}$ such that $h_i = 1$ we have $D^*(\mathbf{x}, \mathbf{h})_i = 1$, for all $\mathbf{x} \in \mathcal{X}$, $i \in \{1, ..., d\}$.
\end{lemma}

The final proposition we state tells us that $\mathbf{H}$ as specified above ensures the generator learns to replicate the desired distribution.

\begin{proposition}	\label{prop:h}
	Suppose $\mathbf{H}$ is defined as above. Then the solution to (\ref{eq:phat}) is unique and satisfies
	\begin{equation}
	\hat{p}(\mathbf{x} | \mathbf{m}_1) = \hat{p}(\mathbf{x} | \mathbf{m}_2)
	\end{equation}
	for all $\mathbf{m}_1, \mathbf{m}_2 \in \{0, 1\}^d$. In particular, $\hat{p}(\mathbf{x} | \mathbf{m}) = \hat{p}(\mathbf{x} | \mathbf{1})$ and since $\mathbf{M}$ is independent of $\mathbf{X}$, $\hat{p}(\mathbf{x} | \mathbf{1})$ is the density of $\mathbf{X}$. The distribution of $\hat{\mathbf{X}}$ is therefore the same as the distribution of $\mathbf{X}$.
\end{proposition}

For the remainder of the paper, $\mathbf{B}$ and $\mathbf{H}$ will be defined as in equations (\ref{eq:bdef}) and (\ref{eq:hdef}).

\section{GAIN Algorithm}\label{sect:gain_algorithm}

Using an approach similar to that in \cite{GAN}, we solve the minimax optimization problem (\ref{eq:obj}) in an iterative manner. Both $G$ and $D$ are modeled as fully connected neural nets.

We first optimize the discriminator $D$ with a fixed generator $G$ using mini-batches of size $k_{D}$\footnote{Details of hyper-parameter selection can be found in the Supplementary Materials.}. For each sample in the mini-batch, $(\tilde{\mathbf{x}}(j), \mathbf{m}(j))$\footnote{The index $j$ now corresponds to the $j$-th sample of the mini-batch, rather than the $j$-th sample of the entire dataset.}, we draw $k_D$ independent samples, $\mathbf{z}(j)$ and $\mathbf{b}(j)$, of $\mathbf{Z}$ and $\mathbf{B}$ and compute $\hat{\mathbf{x}}(j)$ and $\mathbf{h}(j)$ accordingly. Lemma \ref{lem:DH} then tells us that the only outputs of $D$ that depend on $G$ are the ones corresponding to $b_i = 0$ for each sample. We therefore only train $D$ to give us these outputs (if we also trained $D$ to match the outputs specified in Lemma \ref{lem:DH} we would gain no information about $G$, but $D$ would overfit to the hint vector). We define $\mathcal{L}_D : \{0, 1\}^d \times [0, 1]^d \times \{0, 1\}^d \to \mathbb{R}$ by
\begin{align}
\mathcal{L}_D(\mathbf{m}, \hat{\mathbf{m}}, \mathbf{b}) = \sum_{i : b_i = 0} \Big[&m_i\log(\hat{m}_{i})\\\nonumber
&+ (1-m_{i})\log(1-\hat{m}_{i})\Big].
\end{align}
$D$ is then trained according to
\begin{equation} \label{eq:dloss2}
\min_D -\sum_{j=1}^{k_D} \mathcal{L}_D(\mathbf{m}(j), \hat{\mathbf{m}}(j), \mathbf{b}(j))
\end{equation}
recalling that $\hat{\mathbf{m}}(j) = D(\hat{\mathbf{x}}(j), \mathbf{m}(j))$.

Second, we optimize the generator $G$ using the newly updated discriminator $D$ with mini-batches of size $k_{G}$. We first note that $G$ in fact outputs a value for the {\em entire} data vector (including values for the components we observed). Therefore, in training $G$, we not only ensure that the imputed values for missing components ($m_j = 0$) successfully fool the discriminator (as defined by the minimax game), we also ensure that the values outputted by $G$ for observed components ($m_j = 1$) are close to those actually observed. This is justified by noting that the conditional distribution of $\mathbf{X}$ given $\tilde{\mathbf{X}} = \tilde{\mathbf{x}}$ obviously fixes the components of $\mathbf{X}$ corresponding to $M_i = 1$ to be $\tilde{X}_i$. This also ensures that the representations learned in the hidden layers of $\tilde{\mathbf{X}}$ suitably capture the information contained in $\tilde{\mathbf{X}}$ (as in an auto-encoder).

\begin{algorithm}[h!]
	\caption{Pseudo-code of GAIN}\label{alg:pseudo}
	\begin{algorithmic} 
		\WHILE {training loss has not converged}
		\STATE \textbf{(1) Discriminator optimization} 
		\STATE Draw $k_D$ samples from the dataset $\{(\tilde{\mathbf{x}}(j), \mathbf{m}(j))\}_{j=1}^{k_D}$
		\STATE Draw $k_D$ i.i.d. samples, $\{\mathbf{z}(j)\}_{j=1}^{k_D}$, of $\mathbf{Z}$
		\STATE Draw $k_D$ i.i.d. samples, $\{\mathbf{b}(j)\}_{j=1}^{k_D}$, of $\mathbf{B}$
		\FOR {$j = 1, ..., k_D$}
		\STATE $\bar{\mathbf{x}}(j) \gets G(\tilde{\mathbf{x}}(j), \mathbf{m}(j), \mathbf{z}(j))$
		\STATE $\hat{\mathbf{x}}(j) \gets \mathbf{m}(j) \odot \tilde{\mathbf{x}}(j) + (\mathbf{1 - m}(j)) \odot \bar{\mathbf{x}}(j)$
		\STATE $\mathbf{h}(j) = \mathbf{b}(j) \odot \mathbf{m}(j) + 0.5(\mathbf{1 - b}(j))$
		\ENDFOR
		\STATE Update $D$ using stochastic gradient descent (SGD) 
		\begin{equation*}
		\nabla_{D}-\sum_{j=1}^{k_D} \mathcal{L}_D(\mathbf{m}(j), D(\hat{\mathbf{x}}(j), \mathbf{h}(j)), \mathbf{b}(j))
		\end{equation*}
		\STATE \textbf{(2) Generator optimization} 
		\STATE Draw $k_G$ samples from the dataset $\{(\tilde{\mathbf{x}}(j), \mathbf{m}(j))\}_{j=1}^{k_G}$
		\STATE Draw $k_G$ i.i.d. samples, $\{\mathbf{z}(j)\}_{j=1}^{k_G}$ of $\mathbf{Z}$
		\STATE Draw $k_G$ i.i.d. samples, $\{\mathbf{b}(j)\}_{j=1}$ of $\mathbf{B}$
		\FOR {$j = 1, ..., k_G$}
		\STATE $\mathbf{h}(j) = \mathbf{b}(j) \odot \mathbf{m}(j) + 0.5(\mathbf{1 - b}(j))$
		\ENDFOR
		\STATE Update $G$ using SGD (for fixed $D$) 
		\[
		\nabla_{G}\sum_{j=1}^{k_{G}}\mathcal{L}_{G}(\mathbf{m}(j),\hat{\mathbf{m}}(j), \mathbf{b}(j))+\alpha\mathcal{L}_{M}(\mathbf{x}(j),\tilde{\mathbf{x}}(j))
		\]
		\ENDWHILE
	\end{algorithmic} 
\end{algorithm}

    \begin{table*}[t!]
        \renewcommand{\arraystretch}{1.3}
        \caption{Source of gains in GAIN algorithm (Mean $\pm$ Std of RMSE (Gain (\%)))}    
        \label{table:sourceofgain}
        \centering
        \begin{tabular}{ |c|| c | c  |  c| c | c |  }
            \toprule
            \textbf{Algorithm}    & \textbf{Breast} &  \textbf{Spam} & \textbf{Letter}& \textbf{Credit} & \textbf{News}   \\ \midrule
            \textbf{GAIN} & \textbf{.0546 $\pm$ .0006} & \textbf{.0513$\pm$ .0016} & \textbf{.1198$\pm$ .0005}  & \textbf{.1858 $\pm$ .0010} & \textbf{.1441 $\pm$ .0007}  \\   \midrule
            GAIN w/o  & .0701 $\pm$ .0021  & .0676 $\pm$ .0029 & .1344 $\pm$ .0012  & .2436 $\pm$ .0012 & .1612 $\pm$ .0024   \\ 
             $\mathcal{L}_G$& \textbf{(22.1\%)} & \textbf{(24.1\%) }& \textbf{(10.9\%) }& \textbf{(23.7\%)} & \textbf{(10.6\%)} \\  \midrule
            GAIN w/o   & .0767 $\pm$ .0015 & .0672 $\pm$ .0036 & .1586 $\pm$ .0024 & .2533 $\pm$ .0048 & .2522 $\pm$ .0042  \\
            $\mathcal{L}_M$& \textbf{(28.9\%)} & \textbf{(23.7\%) }& \textbf{(24.4\%) }& \textbf{(26.7\%)} & \textbf{(42.9\%)}\\   \midrule
            GAIN w/o  & .0639 $\pm$ .0018 & .0582 $\pm$ .0008 & .1249 $\pm$ .0011  & .2173 $\pm$ .0052 & .1521 $\pm$ .0008   \\ 
            Hint & \textbf{(14.6\%)} & \textbf{(11.9\%)} & \textbf{(4.1\%)} & \textbf{(14.5\%)} & \textbf{(5.3\%)} \\  \midrule
            GAIN w/o   &.0782 $\pm$ .0016  & .0700 $\pm$ .0064& .1671 $\pm$ .0052 & .2789 $\pm$ .0071& .2527 $\pm$ .0052  \\
            Hint \& $\mathcal{L}_M$ & \textbf{(30.1\%)} & \textbf{(26.7\%)} & \textbf{(28.3\%)} & \textbf{(33.4\%)} & \textbf{(43.0\%)}\\
            \bottomrule
        \end{tabular}
    \end{table*}

To achieve this, we define two different loss functions. The first, $\mathcal{L}_G : \{0, 1\}^d \times [0, 1]^d \times \{0, 1\}^d \to \mathbb{R}$, is given by
\begin{align}
 \mathcal{L}_G(\mathbf{m},\hat{\mathbf{m}}, \mathbf{b})&=-\sum_{i : b_i = 0}(1-m_{i})\log(\hat{m}_{i}),
\end{align}
and the second, $\mathcal{L}_M : \mathbb{R}^d \times \mathbb{R}^d \to \mathbb{R}$, by
\begin{align}
 \mathcal{L}_M(\mathbf{x},\mathbf{x}')&=\sum_{i=1}^{d}m_i L_M(x_{i},{x}_{i}'),
\end{align}
where
\[
 L_{M}(x_i,x_i')=\begin{cases}
(x_{i}'-x_{i})^{2}, & \text{if \ensuremath{x_i} is continuous},\\
-x_{i}\log(x_{i}'), & \text{if \ensuremath{x_i} is binary}.
\end{cases}
\]
As can be seen from their definitions, $\mathcal{L}_G$ will apply to the missing components ($m_i = 0$) and $\mathcal{L}_M$ will apply to the observed components ($m_i = 1$).

$\mathcal{L}_{G}(\mathbf{m},{\hat{\mathbf{m}}})$ is smaller when $\hat{m}_i$ is closer to $1$ for $i$ such that $m_i = 0$. That is, $\mathcal{L}_{G}(\mathbf{m},{\hat{\mathbf{m}}})$ is smaller when $D$ is less able to identify the imputed values as being imputed (it falsely categorizes them as observed). $\mathcal{L}_{M}(\mathbf{x},{\tilde{\mathbf{x}}})$ is minimized when the reconstructed features (i.e. the values $G$ outputs for features that were observed) are close to the actually observed features.

$G$ is then trained to minimize the weighted sum of the two losses as follows:
\begin{align*}\label{eq:optG}
\min_{G}\sum_{j=1}^{k_{G}}\mathcal{L}_{G}(\mathbf{m}(j),\hat{\mathbf{m}}(j), \mathbf{b}(j))+\alpha\mathcal{L}_{M}(\tilde{\mathbf{x}}(j), \hat{\mathbf{x}}(j)),
\end{align*}
where $\alpha$ is a hyper-parameter. 

The pseudo-code is presented in Algorithm \ref{alg:pseudo}.

\section{Experiments}\label{sect:experiments}
In this section, we validate the performance of GAIN using multiple real-world datasets. In the first set of experiments we qualitatively analyze the properties of GAIN. In the second we quantitatively evaluate the imputation performance of GAIN using various UCI datasets \cite{UCI}, giving comparisons with state-of-the-art imputation methods. In the third we evaluate the performance of GAIN in various settings (such as on datasets with different missing rates). In the final set of experiments we evaluate GAIN against other imputation algorithms when the goal is to perform prediction on the imputed dataset.

We conduct each experiment 10 times and within each experiment we use 5-cross validations. We report either RMSE or AUROC as the performance metric along with their standard deviations across the 10 experiments. Unless otherwise stated, missingness is applied to the datasets by randomly removing 20\% of all data points (MCAR).

\subsection{Source of gain}
The potential sources of gain for the GAIN framework are: the use of a GAN-like architecture (through $\mathcal{L}_G$), the use of reconstruction error in the loss ($\mathcal{L}_M$), and the use of the hint ($\mathbf{H}$). In order to understand how each of these affects the performance of GAIN, we exclude one or two of them and compare the performances of the resulting architectures against the full GAIN architecture.

Table \ref{table:sourceofgain} shows that the performance of GAIN is improved when all three components are included. More specifically, the full GAIN framework has a 15\% improvement over the simple auto-encoder model (i.e. GAIN w/o $\mathcal{L}_G$). Furthermore, utilizing the hint vector additionally gives improvements of 10\%.

    \begin{table*}[t!]
        \renewcommand{\arraystretch}{1.3}
        \caption{Imputation performance in terms of RMSE (Average $\pm$ Std of RMSE)}    
        \label{table:impute}
        \centering
        \begin{tabular}{ |c|| c | c  |  c| c | c |   }
            \toprule
            \textbf{Algorithm}    & \textbf{Breast} &  \textbf{Spam} & \textbf{Letter}& \textbf{Credit} & \textbf{News}   \\ \midrule
            \textbf{GAIN} & \textbf{.0546 $\pm$ .0006} & \textbf{.0513$\pm$ .0016} & \textbf{.1198$\pm$ .0005}  & \textbf{.1858 $\pm$ .0010} & \textbf{.1441 $\pm$ .0007}  \\  \midrule
            MICE  & .0646 $\pm$ .0028 & .0699 $\pm$ .0010 & .1537 $\pm$ .0006  & .2585 $\pm$ .0011 & .1763 $\pm$ .0007  \\ 
            MissForest  & .0608 $\pm$ .0013 & .0553 $\pm$ .0013 & .1605 $\pm$ .0004  & .1976 $\pm$ .0015 & .1623 $\pm$ 0.012
             \\ 
            Matrix  & .0946 $\pm$ .0020 & .0542 $\pm$ .0006 & .1442 $\pm$ .0006  & .2602 $\pm$ .0073 & .2282 $\pm$ .0005 \\
            Auto-encoder  & .0697 $\pm$ .0018 & .0670 $\pm$ .0030 & .1351 $\pm$ .0009  & .2388 $\pm$ .0005 & .1667 $\pm$ .0014   \\
            EM & .0634 $\pm$ .0021 & .0712 $\pm$ .0012 & .1563 $\pm$ .0012  & .2604 $\pm$ .0015 & .1912 $\pm$ .0011   \\ \bottomrule
        \end{tabular}
    \end{table*}
    
\subsection{Quantitative analysis of GAIN}
We use five real-world datasets from UCI Machine Learning Repository
\cite{UCI} (Breast, Spam, Letter, Credit, and News) to quantitatively evaluate the imputation performance of GAIN. Details of each dataset can be found in the Supplementary Materials.

In table \ref{table:impute} we report the RMSE (and its standard deviation) for GAIN and 5 other state-of-the-art
imputation methods: MICE \cite{MICE,MICE-R}, MissForest \cite{missforest},
Matrix completion (Matrix) \cite{Mat-0}, Auto-encoder \cite{autoencoder}
and Expectation-maximization (EM) \cite{EM}. As can be seen from the table, GAIN significantly outperforms each benchmark. Results for the imputation quality of categorical variables in this experiment are given in the Supplementary Materials.

\subsection{GAIN in different settings}
    \begin{figure*}[t!]
    \centering
    \includegraphics[width=\textwidth]{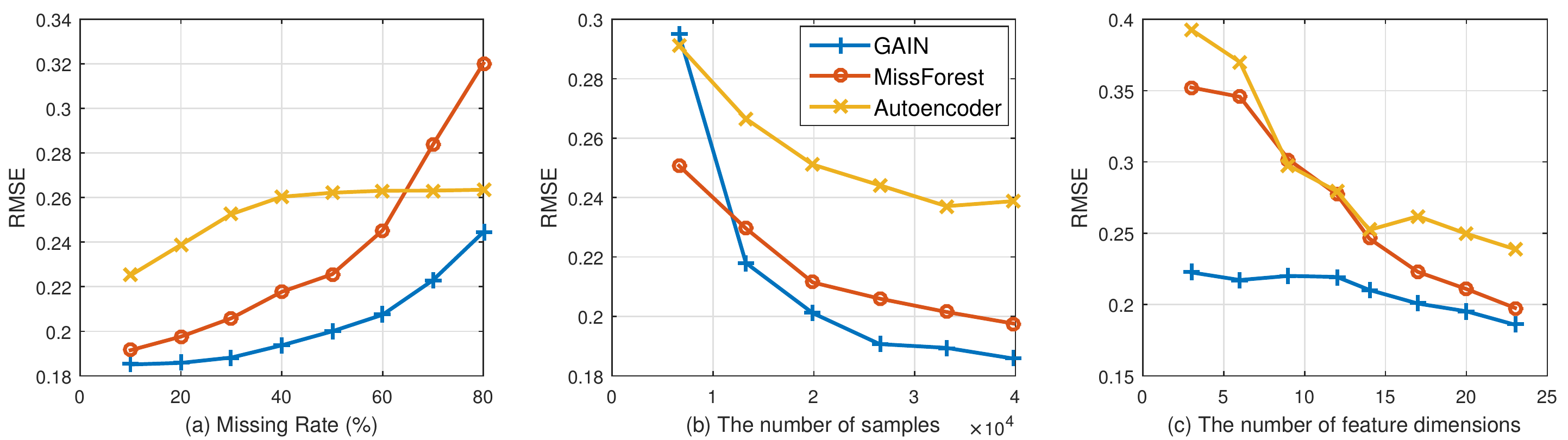}
    \caption{RMSE performance in different settings: (a) Various missing rates, (b) Various number of samples, (c) Various feature dimensions}
    \label{fig:setting}
    \end{figure*}
To better understand GAIN, we conduct several experiments in which we vary the
missing rate, the number of samples, and the number of dimensions
using Credit dataset. Fig. \ref{fig:setting} shows the performance
(RMSE) of GAIN within these different settings in comparison to the two most competitive benchmarks (MissForest and Auto-encoder). Fig. \ref{fig:setting} (a) shows that, even though the performance of each algorithm decreases as missing rates increase, GAIN consistently outperforms the benchmarks across the entire range of missing rates.

Fig. \ref{fig:setting} (b) shows that as the number of samples increases, the performance improvements of GAIN over the benchmarks also increases. This is due to the large number of parameters in GAIN that need to be optimized, however, as demonstrated on the Breast dataset (in Table \ref{table:impute}), GAIN is still able to outperform the benchmarks even when the number of samples is relatively small.

Fig. \ref{fig:setting} (c) shows that GAIN is also robust to the number of feature dimensions. On the other hand, the discriminative model (MissForest) cannot as easily cope when the number of feature dimensions is small.

\subsection{Prediction Performance}
We now compare GAIN against the same benchmarks with respect to the accuracy of post-imputation prediction. For this purpose, we use Area Under the Receiver Operating Characteristic Curve (AUROC) as the measure of performance. To be fair to all methods, we use the same predictive model (logistic regression) in all cases.

Comparisons are made on all datasets except Letter (as it has multi-class labels) and the results are reported in Table \ref{tab:Prediction}.

\begin{table*}[t!]
	\renewcommand{\arraystretch}{1.3}
	\caption{Prediction performance comparison}
	\label{tab:Prediction}
	\centering
	\begin{tabular}{|c|c|c|c|c|}
		\toprule
		\multirow{2}{*}{\textbf{Algorithm}}&\multicolumn{4}{c|}{\textbf{AUROC (Average $\pm$ Std)}}  \\
		\cmidrule{2-5}
		& \textbf{Breast} & \textbf{Spam} &  \textbf{Credit} & \textbf{News} \\
		\midrule
		\textbf{GAIN}& \textbf{.9930 $\pm$ .0073} & \textbf{.9529 $\pm$ .0023}  & \textbf{.7527 $\pm$ .0031} & \textbf{.9711 $\pm$ .0027}  \\
		\midrule
		MICE & .9914 $\pm$ .0034  & .9495 $\pm$ .0031 & .7427 $\pm$ .0026& .9451 $\pm$ .0037\\
		MissForest  & .9860 $\pm$ .0112&.9520 $\pm$ .0061&.7498 $\pm$ .0047& .9597 $\pm$ .0043 \\
		Matrix   & .9897 $\pm$ .0042 & .8639 $\pm$ .0055 & .7059 $\pm$ .0150 & .8578 $\pm$ .0125  \\
		Auto-encoder& .9916 $\pm$ .0059 & .9403 $\pm$ .0051 & .7485 $\pm$ .0031 & .9321 $\pm$ .0058 \\
		EM& .9899 $\pm$ .0147 & .9217 $\pm$ .0093 & .7390 $\pm$ .0079 & .8987 $\pm$ .0157 \\
		\bottomrule
	\end{tabular}
\end{table*}

As Table \ref{tab:Prediction} shows, GAIN, which we have already shown to achieve the best imputation accuracy (in Table \ref{table:impute}), yields the best post-imputation prediction accuracy. However, even in cases where the improvement in imputation accuracy is large, the  improvements in prediction accuracy are not always significant. This is probably due to the fact that there is sufficient information in the (80\%) observed data to predict the label.

\begin{figure}[t!]
	\center
	\includegraphics[width=0.5\textwidth]{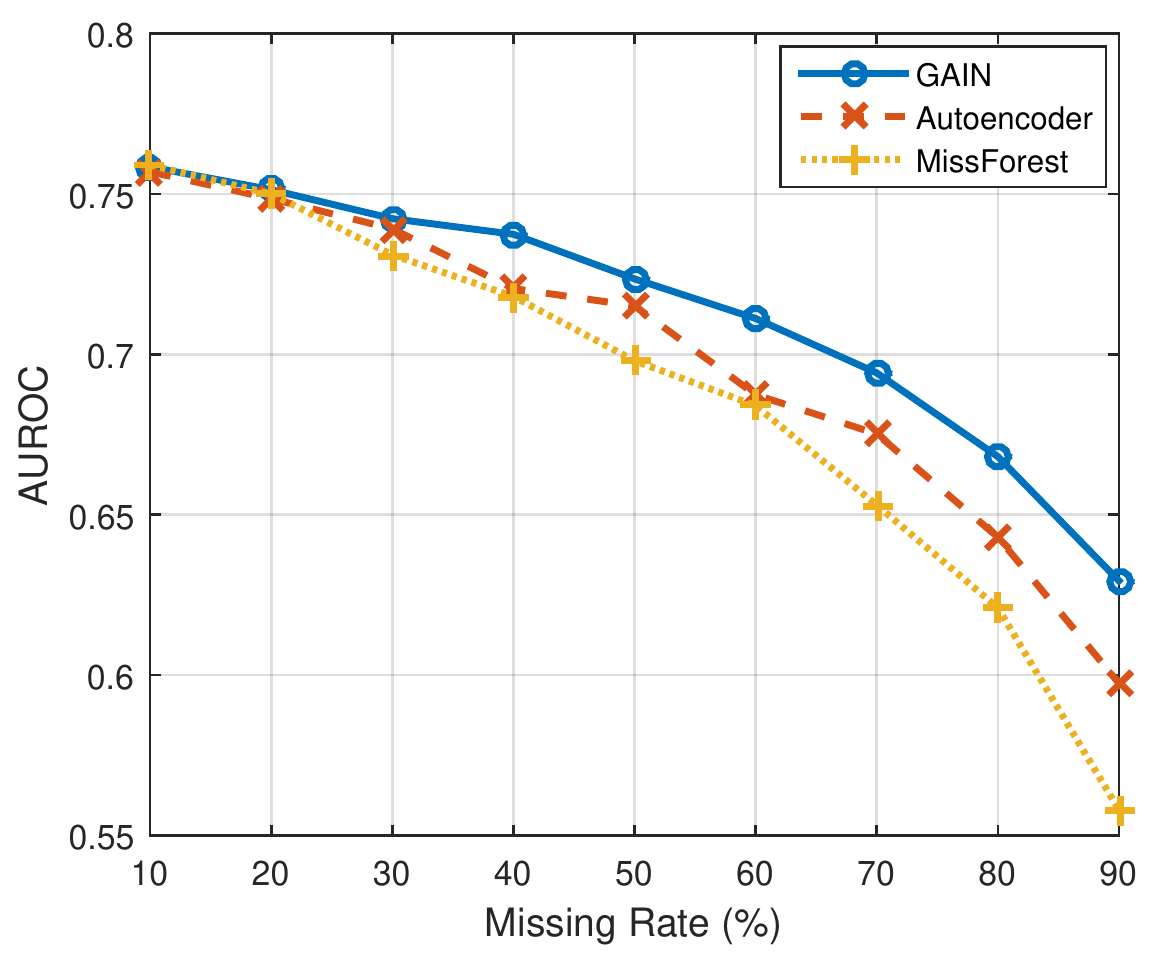}
	\caption{The AUROC performance with various missing rates with Credit dataset}
	\label{fig:prediction_missing_rate}
\end{figure}

\textbf{Prediction accuracy with various missing rates: }In this experiment, we evaluate the post-imputation prediction performance when the missing rate of the dataset is varied. Note that every dataset (except Letter) has their own binary label.

The results of this experiment (for GAIN and the two most competitive benchmarks) are shown in Fig. \ref{fig:prediction_missing_rate}. In particular, the performance of GAIN is significantly better than the other two for higher missing rates, this is due to the fact that as the information contained in the observed data decreases (due to more values being missing), the imputation quality becomes more important, and GAIN has already been shown to provide (significantly) better quality imputations.

\subsection{Congeniality of GAIN}

The congeniality of an imputation model is its ability to impute values that respect the feature-label relationship \cite{congeniality_meng, reason_of_imputation1, reason_of_imputation2}. The congeniality of an imputation model can be evaluated by measuring the effects on the feature-label relationships after the imputation. We compare the logistic regression parameters, $\mathbf{w}$, learned from the complete Credit dataset with the parameters, $\hat{\mathbf{w}}$, learned from an incomplete Credit dataset by first imputing and then performing logistic regression.

We report the mean and standard deviation of both the mean bias $(||\mathbf{w} - \hat{\mathbf{w}}||_1)$ and the mean square error $(||\mathbf{w} - \hat{\mathbf{w}}||_2)$ for each method in Table \ref{tab:congeniality}. These quantities being lower indicates that the imputation algorithm better respects the relationship between feature and label. As can be seen in the table, GAIN achieves significantly lower mean bias and mean square error than other state-of-the-art imputation algorithms (from 8.9\% to 79.2\% performance improvements).

\begin{table}[t!]
	\renewcommand{\arraystretch}{1.3}
	\caption{Congeniality of imputation models}
	\label{tab:congeniality}
	\centering
	\begin{tabular}{|c|c|c|}
		\toprule
		\multirow{2}{*}{\textbf{Algorithm}} & \textbf{Mean Bias} & \textbf{MSE} \\
		& $(||\mathbf{w} - \hat{\mathbf{w}}||_1)$ & $(||\mathbf{w} - \hat{\mathbf{w}}||_2)$\\
		\midrule
		\textbf{GAIN} & \textbf{0.3163$\pm$ 0.0887} & \textbf{0.5078$\pm$ 0.1137} \\ \midrule
		MICE   & 0.8315 $\pm$ 0.2293 & 0.9467 $\pm$ 0.2083 \\
		MissForest & 0.6730 $\pm$ 0.1937 & 0.7081 $\pm$ 0.1625 \\
		Matrix & 1.5321 $\pm$ 0.0017 & 1.6660 $\pm$ 0.0015 \\
		Auto-encoder & 0.3500 $\pm$ 0.1503 & 0.5608 $\pm$0.1697  \\
		EM  & 0.8418 $\pm$ 0.2675 & 0.9369 $\pm$ 0.2296 \\
		\bottomrule
	\end{tabular}
\end{table}

\section{Conclusion}\label{sect:conclusion}
We propose a generative model for missing data imputation, GAIN. This novel architecture generalizes the well-known GAN such that it can deal with the unique characteristics of the imputation problem. Various experiments with real-world datasets show that GAIN significantly outperforms state-of-the-art imputation techniques. The development of a new, state-of-the-art technique for imputation can have transformative impacts; most datasets in medicine as well as in other domains have missing data. Future work will investigate the performance of GAIN in recommender systems, error concealment as well as in active sensing \cite{activesensing}. Preliminary results in error concealment using the MNIST dataset \cite{mnist} can be found in the Supplementary Materials - see Fig. 4 and 5. 

\newpage
\section*{Acknowledgement}
The authors would like to thank the reviewers for their helpful comments. The research presented in this paper was supported by the Office of Naval Research (ONR) and the NSF (Grant number: ECCS1462245, ECCS1533983, and ECCS1407712).

\end{document}